\documentclass{article} 
\usepackage{iclr2027_conference,times}

\usepackage{amsmath,amsfonts,bm}

\def\eqref#1{equation~\ref{#1}}

\def\1{\bm{1}}

\DeclareMathAlphabet{\mathsfit}{\encodingdefault}{\sfdefault}{m}{sl}
\SetMathAlphabet{\mathsfit}{bold}{\encodingdefault}{\sfdefault}{bx}{n}

\usepackage{hyperref}
\usepackage{url}
\title{CoeF-SFL: Preserving Collaborative Server-Client Learning with Enhanced Communication Efficiency}

\author{
Junwoo Bae \\
Myongji University \\
\texttt{ag660340@mju.ac.kr}
\And
Jin-Hyun Ahn \\
Myongji University \\
\texttt{wlsgus3396@mju.ac.kr}
}

\usepackage{amsmath,amssymb,amsthm}
\allowdisplaybreaks   
\usepackage{booktabs}
\usepackage{multirow}
\usepackage{array}
\usepackage{comment}
\usepackage{graphicx}
\usepackage{placeins}
\usepackage{capt-of}
\usepackage{tabularx}
\usepackage{algorithm}
\usepackage{algorithmic}
\renewcommand{\floatpagefraction}{0.85}
\renewcommand{\topfraction}{0.9}

\renewcommand{\textfraction}{0.05}
\hypersetup{pdfauthor={},pdfsubject={},pdfkeywords={},colorlinks=true,linkcolor=black,citecolor=blue,urlcolor=blue}

\newtheorem{proposition}{Proposition}
\newtheorem{lemma}{Lemma}

\theoremstyle{definition}
\newtheorem{assumption}{Assumption}

\iclrfinalcopy 
\begin{document}
\frenchspacing

\maketitle

\begin{abstract}
Split Federated Learning (SFL) enables resource-constrained clients to participate in collaborative training, but vanilla SFL exchanges smashed data and gradients at every batch, which incurs significant communication overhead. Recent methods reduce this overhead with an auxiliary network at the client-side cut layer. However, we identify that this approach makes the client optimize a local objective that differs from the end-to-end objective, which fundamentally limits the collaborative training between the client and the server. We propose Compensated Feedback based SFL (CoeF-SFL), a communication-efficient framework that retains the end-to-end objective without any auxiliary network. In CoeF-SFL, the client and the server exchange the smashed data and the gradients once per round and reuse them during local training. Since this reuse makes the gradients stale on the client side, we compensate them with a curvature-based correction in the activation space and develop two variants. CoeF-D approximates the Hessian with a diagonal gradient outer product, while CoeF-J exploits the tractable Jacobian-based Hessian of a surrogate loss that upper-bounds the true loss. We provide the theoretical background of each method, characterizing its compensation. Across vision and language tasks, model capacities, cut layers, and data distributions, CoeF-SFL significantly outperforms auxiliary-network-based methods under the same communication frequency, and the improvement is most substantial on vision tasks. Code is available at \url{https://anonymous.4open.science/r/CoeF-SFL-2686/README.md}
\end{abstract}

\providecommand{\std}[1]{\,\mbox{\tiny$\pm$\,#1}}
\vspace{-1mm}
\section{Introduction} 
Split Federated Learning (SFL) \citep{thapa2022splitfed} enables resource-constrained clients to collaboratively train a model without sharing raw data. By splitting the model at a cut layer as in split learning \citep{gupta2018splitlearning,vepakomma2018splitnn} and training the client-side models in parallel as in federated learning \citep{mcmahan2017fedavg}, SFL avoids both the heavy client-side computation of federated learning and the sequential training latency of split learning. In SFL, the server updates the server-side model with the smashed data from the clients and returns the corresponding gradients, so that the split models collaboratively optimize the end-to-end (E2E) objective through server-side feedback.

\subsection{Auxiliary-network-based SFL} 
Although SFL addresses computational burden and latency, significant communication overhead remains due to the per-batch exchange between the client and the server. To address this overhead, \textit{AccSFL} \citep{han2022localloss} has introduced an auxiliary network into the SFL framework. Specifically, by attaching the network at the cut layer, the client-side models can perform parameter updates without any feedback from the server side. This mechanism eliminates the need to wait for the gradient corresponding to the transmitted smashed data, thereby reducing latency and communication burden. In fact, since AccSFL was proposed, many studies \citep[e.g.,][]{oh2022locfedmix, mu2023csefsl, nair2025fslsage,shin2026fedsplitx} have assumed a system model in which an auxiliary network is attached to the client side to achieve high communication efficiency.

\begin{table}[t]
\centering
\small
\setlength{\tabcolsep}{6pt}
\renewcommand{\arraystretch}{1.1}
\caption{Preliminary results of vanilla SFL \citep{thapa2022splitfed}, AccSFL \citep{han2022localloss}, two reference methods, reported as test accuracy (\%) and EM/F1 (\%) on vision and language tasks with ViT-Tiny and DistilRoBERTa, respectively. Gradient reuse (w/o comp.) is the naive baseline, and Gradient reuse (w/ perfect comp.) is the oracle. For all methods except vanilla SFL, each round, i.e., the interval between two consecutive server feedbacks, corresponds to one local epoch. The other settings, including the SFL setting, local training, and hyperparameter tuning, are described in Section~\ref{sec:exp}.}
\label{tab:intro}
\begin{tabular}{l c c c}
\toprule
\textbf{Method} & \textbf{CIFAR-100} & \textbf{Food-101} & \textbf{Tiny-ImageNet} \\
\midrule

Vanilla SFL & $85.0\std{0.3}$ & $83.8\std{0.1}$ & $78.6\std{0.2}$ \\
AccSFL & $65.6\std{1.2}$ & $67.5\std{0.1}$ & $47.4\std{0.2}$ \\
\midrule

Gradient reuse (w/o comp.) & $78.4\std{0.4}$ & $75.3\std{2.1}$ & $72.8\std{0.2}$ \\
Gradient reuse (w/ perfect comp.) & $84.5\std{0.3}$ & $83.3\std{0.0}$ & $78.3\std{0.3}$ \\

\bottomrule
\end{tabular}\\[6pt]
{\setlength{\tabcolsep}{3pt}
\begin{tabular}{l c c c c c}
\toprule
\textbf{Method} & \textbf{20 Newsgroups} & \textbf{MNLI-m} & \textbf{MNLI-mm} & \textbf{EM} & \textbf{F1} \\
\midrule

Vanilla SFL & $69.8\std{0.2}$ & $83.3\std{0.3}$ & $83.7\std{0.3}$ & $71.3\std{0.2}$ & $74.5\std{0.2}$ \\
AccSFL & $66.6\std{0.3}$ & $77.7\std{0.3}$ & $78.5\std{0.3}$ & $61.6\std{1.2}$ & $64.7\std{1.1}$ \\
\midrule

Gradient reuse (w/o comp.) & $69.3\std{0.0}$ & $37.1\std{6.2}$ & $37.7\std{7.1}$ & $49.9\std{0.2}$ & $50.0\std{0.2}$ \\
Gradient reuse (w/ perfect comp.) & $69.3\std{0.2}$ & $83.0\std{0.1}$ & $83.1\std{0.2}$ & $70.2\std{2.1}$ & $73.1\std{1.5}$ \\

\bottomrule
\end{tabular}}
\vspace{-8pt}
\end{table}

However, the auxiliary-network-based approach sacrifices a key advantage of the SFL framework, which is that the split models collaboratively optimize the E2E objective through server-side feedback. Although \citep{nair2025fslsage} pointed out this limitation, noting that such methods ``lack server feedback and potentially suffer poor accuracy,'' why it constitutes a fundamental problem can be explained by the analysis of decoupled training in \citep{wang2021infopro}, which corresponds to auxiliary-network-based SFL with a single client. Their information-theoretic and empirical analysis shows that such decoupled training with a local objective is suboptimal. Specifically, the local objective drives the features at the cut layer to be over-compressed, which discards task-relevant information that the subsequent layers would otherwise exploit. Furthermore, once this information is discarded, the subsequent layers cannot recover it even with additional capacity. Our preliminary results in Table~\ref{tab:intro} support this explanation in the context of SFL, showing the corresponding performance degradation of AccSFL compared with vanilla SFL on various vision and language tasks.

\subsection{Motivations and Contributions}
Motivated by this observation, we argue that the client-side model should be trained with the gradient computed by the server, as in vanilla SFL, while preserving the communication efficiency. To this end, the clients in the proposed method reuse the gradients received at the beginning of the round over multiple local updates, just as the server repeatedly updates its model with the received smashed data in auxiliary-network-based SFL. Although this retains the E2E objective with improved communication efficiency, it causes a mismatch between the activation and the gradient, since the activation of each input changes as the client-side model is updated while the received gradient remains fixed. We refer to this mismatch as the \textit{staleness} of the activation-space gradient, and formulate it after the $i$-th local update for input $x$ as
\begin{equation}
\label{eq:gradient_staleness}
\varepsilon_{i}(x)
:=
\underbrace{\nabla_s\mathcal{L}\bigl(s_{i};\theta^s\bigr)}_{\text{(a) Fresh gradient}}
-\underbrace{\nabla_s\mathcal{L}\bigl(s_0;\theta^s\bigr)}_{\text{(b) Stale gradient}},
\qquad s_i=f_c\bigl(x;\theta^{c}_{i}\bigr),
\end{equation}
where $s_i$ is the activation of $x$ after the $i$-th local update, and $\theta^s$ is the server-side model at the beginning of the round. The fresh gradient matches the current client-side model $\theta^{c}_{i}$ but is unavailable to the client, which does not hold the server-side model, whereas the stale gradient is computed by the server at the initial activation $s_0$. As shown by Gradient reuse (w/o comp.) in Table~\ref{tab:intro}, which naively reuses the received gradients, this staleness severely degrades the performance of SFL.

To examine the potential of compensation, Gradient reuse (w/ perfect comp.) in Table~\ref{tab:intro} denotes an oracle that completely removes the staleness by using the fresh gradient computed with the server-side model at the beginning of the round. It significantly outperforms both AccSFL and Gradient reuse (w/o comp.) and achieves performance close to that of vanilla SFL, which suggests that SFL can retain the communication efficiency while outperforming auxiliary-network-based SFL once the stale gradients are properly compensated. Building on this observation, we propose Compensated Feedback based SFL (CoeF-SFL), which compensates the stale gradients with a curvature-based correction in the activation space and thus achieves both collaborative server-client learning and communication efficiency, as illustrated in Figure~\ref{fig:overview}. We develop two variants, CoeF-D and CoeF-J. CoeF-D approximates the Hessian with a diagonal gradient outer product and requires no additional server-side computation, while CoeF-J exploits the tractable Jacobian-based Hessian of a surrogate loss function that upper-bounds the true loss, which requires only first-order computation on the server but tends to be more robust under non-IID settings and deep cut layers, as shown in Section~\ref{sec:exp}. The design space that we explored for the compensation is discussed in Section~\ref{sec:method}. Beyond the theoretical and empirical validation of CoeF-SFL, the contributions of this paper are summarized as follows.

\begin{figure}[t]
    \centering
    \includegraphics[width=1.0\textwidth]{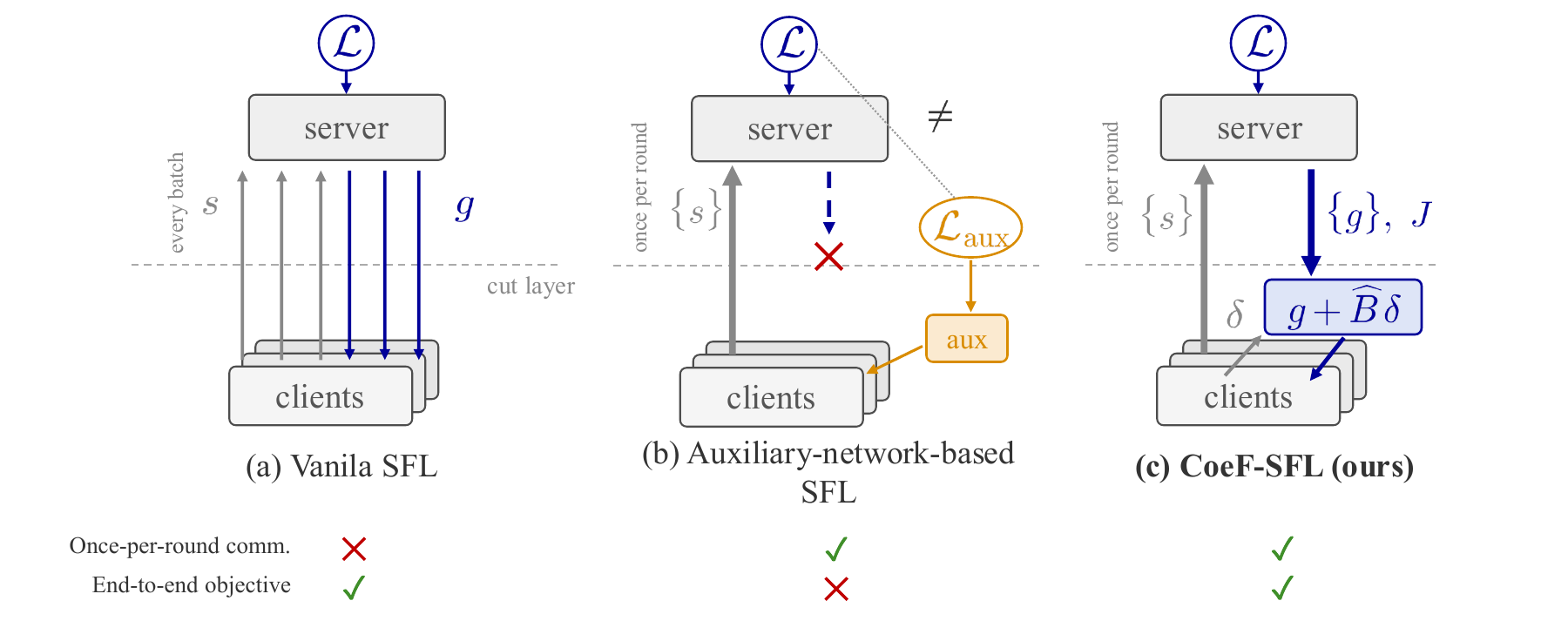}
\caption{Comparison of SFL frameworks. Vanilla SFL optimizes the end-to-end objective $\mathcal{L}$ through server-side feedback at every batch, which incurs heavy communication. Auxiliary-network-based SFL communicates once per round, but the client optimizes an auxiliary objective $\mathcal{L}_{\mathrm{aux}}$ that differs from $\mathcal{L}$. CoeF-SFL communicates once per round and preserves the end-to-end objective, compensating the received gradients according to the displacement $\delta$ of activation.}
\label{fig:overview}
\vspace{-8pt}
\end{figure}

\vspace{-3mm}
\paragraph{Identifying the loss of the end-to-end objective.} We identify that auxiliary-network-based SFL, which has been widely adopted in recent SFL studies, loses the E2E objective that the client-side and server-side models collaboratively optimize. To retain this objective with high communication efficiency, we propose a new SFL framework in which the clients reuse the gradients received from the server over multiple local updates.

\vspace{-3mm}
\paragraph{Analyzing and compensating gradient staleness.} To the best of our knowledge, this is the first work that analyzes the staleness of the activation-space gradient, which arises when the client-side model is updated with the gradients received at the beginning of the round. We analyze this staleness both empirically and theoretically, and develop two compensation methods, CoeF-D and CoeF-J, to overcome it.

\vspace{0mm}

\section{Related Work}
\label{sec:related_work}
\paragraph{Auxiliary-network-based SFL.}
Several variants have been built on AccSFL \citep{han2022localloss}, which maintains a separate server-side model for each client. LocFedMix \citep{oh2022locfedmix} additionally introduces data mixup, and CSE-FSL \citep{mu2023csefsl} adopts a single server-side model to reduce the storage requirement of the server. FSL-SAGE \citep{nair2025fslsage} addresses the lack of server feedback by periodically aligning the auxiliary models with the server-side model on the server, and FedSplitX \citep{shin2026fedsplitx} extends the auxiliary-network-based mechanism to environments with system heterogeneity. Since we consider the setting of SplitFed V1 \citep{thapa2022splitfed} without system-level heterogeneity across clients, and data mixup is orthogonal to the update rule, FSL-SAGE is the only variant whose update rule substantially differs from that of AccSFL in our setting. Therefore, we adopt AccSFL and FSL-SAGE as the baselines of auxiliary-network-based SFL.

\paragraph{Straggler-resilient SFL.}

Unlike auxiliary-network-based SFL, GAS \citep{yang2024gas} and MU-SplitFed \citep{liang2025straggler} consider that clients upload their activations at different times owing to system heterogeneity, and aim to prevent training from being delayed by stragglers. In GAS, each client receives the gradient at every batch as in vanilla SFL, but the gradient is computed with the server-side model at the time of the upload and returned immediately. In MU-SplitFed, the clients perform zeroth-order optimization without backpropagation, and the server performs multiple updates for each client update while returning a scalar feedback instead of the gradient. Although neither method reduces the upload frequency of the clients, which is the focus of this paper, we include them for a more informative comparison by assuming, only for these two methods, that the clients upload their activations sequentially. 

\vspace{-3mm}
\paragraph{Delay compensation in asynchronous SGD.} DC-ASGD \citep{zheng2017dcasgd} compensates delayed parameter gradients in asynchronous SGD through a first-order expansion, where the Hessian
is approximated by a scaled diagonal outer product of the gradient and multiplied by the parameter
displacement. We draw on this principle for a different learning structure. In SFL with gradient
reuse, the stale quantity is the activation-space gradient rather than the parameter gradient, its staleness
is caused by the client’s own local updates rather than by the asynchrony across workers, and
the client has to compensate it without access to the server-side model. CoeF-D carries the compensation
of DC-ASGD from the parameter space to the activation space using only the received
gradients, while CoeF-J additionally exploits the Jacobian computed by the server at the beginning
of the round.

\vspace{-3mm}
\paragraph{Curvature surrogates and upper-bound models.}
Curvature information in centralized training can be accessed through the generalized Gauss--Newton (GGN) matrix \citep{schraudolph2002ggn,martens2020natural}, stochastic diagonal estimators \citep{bekas2007diag}, and low-rank representations \citep{dangel2022vivit}. The gradient outer product offers another surrogate, which coincides with the empirical Fisher and generally differs from the Hessian \citep{kunstner2019empiricalfisher}. Majorization--minimization \citep{hunter2004mm,bohning1988} instead minimizes an upper-bound model of the objective, for which the global curvature bound of the softmax cross-entropy loss \citep{bohning1992} is a well-known ingredient. These methods estimate the curvature with respect to the model parameters, whereas the compensation in our setting requires the curvature with respect to the cut-layer activation, which the client cannot compute without the server-side model. CoeF-D uses the diagonal gradient outer product as its surrogate, and CoeF-J replaces the output-loss curvature in the GGN structure with its global bound, which yields an upper-bound model whose Hessian is computed from the Jacobian transmitted by the server.

\section{Compensated Feedback based SFL (CoeF-SFL)}
\label{sec:method}

\vspace{-3mm}
\paragraph{Problem setup.} We consider a SplitFed V1 setting with $K$ clients and $T$ rounds. For simplicity, we omit the indices of clients and rounds, and denote the client-side and server-side models at the beginning of each round by $\theta^{c}$ and $\theta^s$, respectively. At the beginning of the round, the client generates the smashed data $\bm{s}=\left\{f_c\bigl(x;\theta^{c}\bigr)\right\}$ from its local data $\mathcal{D}=\left\{(x,y)\right\}$ through the client-side model. The server receives $\left\{(s,y)\right\}$ and computes the gradients $\bm{g}=\left\{\nabla_s\mathcal{L}(s;\theta^s)\right\}$ with the server-side model. After the client receives $\bm{g}$, the client and the server perform $I$ local updates with gradient reuse and smashed-data reuse, respectively.

Starting with $\theta_{0}^{c}=\theta^{c}$, the client-side model after the $i$-th local update is denoted by $\theta_{i}^{c}$. Although the data in the mini-batch differ across local updates, we omit the corresponding index for simplicity. Accordingly, $\bm{s}_{i}=\left\{f_c\bigl(x;\theta^{c}_{i}\bigr)\right\}$ denotes the smashed data recomputed with $\theta^{c}_{i}$, and $\bm{g}_{0,i}$ denotes the received gradients $\bm{g}$, both for the data in the mini-batch of the $i$-th local update. Then, with Gradient reuse (w/o comp.), the client-side model is updated as
\begin{equation}\label{eq:client-local-stale}
    \theta_i^{c}=\theta_{i-1}^{c}-\eta J_{i-1}^{\top}\bm{g}_{0,i-1},
\qquad J_{i}:=\frac{\partial \bm{s}_{i}}{\partial\theta^{c}_{i}},
\end{equation}
whereas with Gradient reuse (w/ perfect comp.), it is updated as
\begin{equation}
    \theta_i^{c}=\theta_{i-1}^{c}-\eta J_{i-1}^{\top}\nabla_s\mathcal{L}\bigl(\bm{s}_{i-1};\theta^s\bigr).
\end{equation}
Furthermore, for AccSFL without knowledge distillation, 
\begin{equation}
    \theta_i^{c}=\theta_{i-1}^{c}-\eta \nabla_{\theta^c}\mathcal{L}_{\mathrm{aux}}\bigl(\bm{s}_{i-1};\theta^c_{i-1}\bigr),
\end{equation}
where $\mathcal{L}_{\mathrm{aux}}$ denotes the local loss computed through the auxiliary network.
In all cases, starting with $\theta_{0}^{s}=\theta^{s}$, the server-side model is updated as
\begin{equation}\label{eq:server-local}
    \theta_i^{s}=\theta_{i-1}^{s}-\eta \nabla_{\theta^{s}}\mathcal{L}\bigl(\bm{s}_{0,i-1};\theta_{i-1}^s\bigr),
\end{equation}
where $\bm{s}_{0,i}$ denotes the smashed data $\bm{s}$ received from the client for the data in the mini-batch of the $i$-th local update. While the server-side update in \eqref{eq:server-local} has no staleness since the received smashed data are fixed, we aim to compensate the staleness of the client-side update in \eqref{eq:client-local-stale}.

\vspace{-3mm}
\paragraph{Toward compensating gradient staleness.}
To resolve this staleness, we explored various approaches. The first approach is to train a regression model on the server, which maps the received smashed data to the computed gradients, and to broadcast it to the clients. Since this model learns the local landscape of the server-side loss at every round, it can provide the gradient even when the activation is displaced during local training. However, our exploration shows that the target of the regression, i.e., the gradient, has a very high dimension, which makes the regression highly inefficient to train. Consequently, the size of the regression network and the per-round training on the server become impractically burdensome. To mitigate this, we further considered techniques such as MAML \citep{finn2017maml, finn2019online} and the neural tangent kernel (NTK) \citep{jacot2018neural}, but they did not sufficiently alleviate the burden. Therefore, we adopt compensation methods based on the first-order expansion of the gradient, i.e., the curvature, in the activation space, which can be regarded as the most fundamental approach.



With the displacement $\bm{\delta}_i:=\bm{s}_i-\bm{s}_{0,i}$, the fundamental theorem of calculus \citep{nocedal2006numerical} gives the fresh gradient exactly,
\begin{equation}
\label{eq:identity}
\bm{g}_i=\bm{g}_{0,i}+\bar H_i\bm{\delta}_i,
\qquad
\bar H_i:=\int_0^1\nabla_s^2\mathcal L\big(\bm{s}_{0,i}+\tau\bm{\delta}_i;\theta^s\big)\,d\tau.
\end{equation}
However, the client cannot compute $\bar H_i$, CoeF-SFL replaces the curvature with matrix $\widehat{B}_i$ built from what the server transmitted at the round-initial exchange, and the client uses
\begin{equation}
\label{eq:corr}
\widetilde {\bm{g}}_i:=\bm{g}_{0,i}+\widehat B_i\bm{\delta}_i =\bm{g}_i+\big(\widehat B_i-\bar H_i\big)\bm{\delta}_i ,
\end{equation}
in place of $\bm{g}_{0,i}$ in \eqref{eq:client-local-stale}. Therefore, the key problem is the reducing the curvature mismatch or the displacement.

\vspace{-3mm}
\paragraph{CoeF-D: diagonal gradient outer product.}
Since CoeF-D computes $\widehat{B}_{\mathrm{D}}$ only from the received gradient, it does not require any further feedback from the server. As discussed in Section~\ref{sec:related_work}, CoeF-D carries the compensation principle of DC-ASGD \citep{zheng2017dcasgd} from the parameter space to the activation space. Specifically, CoeF-D scales the diagonal of the gradient outer product by a compensation factor $\lambda$ as
\begin{equation}
\label{eq:coefd}
\widetilde{\bm{g}}_{\mathrm D,i}=\bm{g}_{0,i}+\widehat B_{\mathrm {D},i}\bm{\delta}_i,
\qquad
\widehat B_{\mathrm {D},i}:=\lambda\operatorname{diag}\big(\bm{g}_{0,i}\odot\bm{g}_{0,i}\big),
\end{equation}
where $\widetilde{\bm{g}}_{\mathrm D,i}$ is the compensated gradient of CoeF-D, and $\lambda>0$ is the compensation factor whose choice is analyzed in Section~\ref{sec:theory}. The outer product is a reasonable surrogate of the Hessian, since it shares the structure of the curvature that it replaces. To show this, let $J_{s,i}:=\partial f_s\bigl(\bm{s}^{*}_{0,i};\theta^{s}\bigr)/\partial s$ denote the Jacobian of the server output $\phi$ computed at $\bm{s}^{*}_{0,i}$, and let $r_i:=\nabla_\phi\mathcal L\bigl(\bm{s}^{*}_{0,i};\theta^s\bigr)$ and $\Lambda_i:=\nabla_\phi^2\mathcal L\bigl(\bm{s}^{*}_{0,i};\theta^s\bigr)$ denote the gradient and the Hessian of the loss with respect to $\phi$, respectively. Here, $\bm{s}^{*}_{0,i}$ denotes the subset of $\bm{s}_0$ whose data belong to the mini-batch of the $i$-th local update at the client, whereas $\bm{s}_{0,i}$ corresponds to the mini-batch of the $i$-th local update at the server. By the chain rule, we have
\begin{equation}
\label{eq:outer}
\bm{g}_{0,i}\odot \bm{g}_{0,i}=J_{s,i}^{\top}\big(r_i r_i^{\top}\big)J_{s,i},
\end{equation}
which has the same form as the generalized Gauss--Newton (GGN) matrix $G_i:= J_{s,i}^{\top}\Lambda_{i}J_{s,i}$ and agrees with it in expectation up to the mean-reduction weight, as discussed in Proposition~\ref{prop1}. The correction is computed elementwise on the client at a cost of $O(d)$ per sample, where $d$ denotes the dimension of the activation space, without any additional server-side computation.

\vspace{-3mm}
\paragraph{CoeF-J: compensation with an upper bound on the loss.}
The error of the compensated gradient in \eqref{eq:corr}, $\widetilde{\bm{g}}_i-\bm{g}_i=(\widehat B_i-\bar H_i)\bm{\delta}_i$, is the product of the curvature mismatch and the displacement. CoeF-D reduces this error by approximating the curvature, which is accurate only while the displacement is small. CoeF-J takes a different route. Instead of approximating the curvature, it constructs a surrogate that upper-bounds the loss, and uses the gradient of the surrogate as the compensated gradient, which does not require an accurate curvature.

The surrogate is built from the structure of the curvature. Since the loss depends on the smashed data only through the server output, its Hessian with respect to the smashed data consists of the GGN term $J_{s,i}^{\top}\Lambda_i J_{s,i}$ and a term that involves the second derivative of the server-side model. The client cannot compute $\Lambda_i$, but for standard losses, the curvature with respect to the server output is bounded as $\Lambda_i\preceq MI$, where $M$ is fixed by the loss, e.g., $M=1/2$ for the softmax cross-entropy loss \citep{bohning1992}. Replacing $\Lambda_i$ with this bound yields
\begin{equation}
\label{eq:major}
\mathcal L(\bm{s}_{0,i}+\bm{\delta}_i;\theta^s)
\le\widehat F_{\mathrm J}(\bm{\delta}_i)
:=\underbrace{\mathcal L(\bm{s}_{0,i};\theta^s)+\langle\bm{g}_{0,i},\bm{\delta}_i\rangle}_{\text{first-order expansion}}
+\underbrace{\tfrac{M}{2}\,\bm{\delta}_i^{\top}P_i\bm{\delta}_i}_{\text{curvature bound}},
\qquad
P_i:=J_{s,i}^{\top}J_{s,i},
\end{equation}
which holds exactly when the server output changes linearly along the displacement, and approximately otherwise, as analyzed in Proposition~\ref{prop2}. Since the surrogate is tangent to the loss at the transmitted smashed data, any step that lowers $\widehat F_{\mathrm J}$ below its value at $\bm{\delta}_i=\bm{0}$ also keeps the loss below its value at $\bm{s}_{0,i}$, as in majorization--minimization \citep{hunter2004mm}. CoeF-J uses the gradient of the surrogate as the compensated gradient,
\begin{equation}
\label{eq:coefj}
\widetilde{\bm{g}}_{\mathrm J,i}=\nabla\widehat F_{\mathrm J}(\bm{\delta}_i)=\bm{g}_{0,i}+\widehat B_{\mathrm J,i}\bm{\delta}_i,
\qquad
\widehat B_{\mathrm J,i}:=MP_i,
\end{equation}
which takes the form of \eqref{eq:corr}. Therefore, the client needs only the gradient and Jacobian computed at the server with $\bm{s}$. In practice, the server does not send Jacobian, whose size grows with the output dimension. Instead, it sends $R$ random projections, whose outer products equal $P_i$ in expectation, as will be discussed in Section \ref{sec:theory}.

\vspace{-3mm}
\paragraph{CoeF-J: interpretation as sensitivity.}
The matrix $P_i$ has an intuitive meaning. To first order, a displacement $\bm{\delta}_i$ changes the server output by $J_{s,i}\bm{\delta}_i$, and thus
\begin{equation}
\label{eq:sens}
\bm{\delta}_i^{\top}P_i\bm{\delta}_i=\lVert J_{s,i}\bm{\delta}_i\rVert^2
\end{equation}
measures how strongly the displacement changes the server output. We call $P_i$ the sensitivity matrix, since \eqref{eq:sens} is large for a displacement that changes the server output strongly and zero for a displacement that leaves it unchanged. With $\bm{g}_{0,i}=J_{s,i}^{\top}r_i$, the compensated gradient can be written as $\widetilde{\bm{g}}_{\mathrm J,i}=J_{s,i}^{\top}(r_i+MJ_{s,i}\bm{\delta}_i)$, where $r_i+MJ_{s,i}\bm{\delta}_i$ is a first-order estimate of the residual at the current smashed data in which $\Lambda_i$ is replaced by its bound $MI$. Since $\langle\bm{\delta}_i,\widehat B_{\mathrm J,i}\bm{\delta}_i\rangle=M\lVert J_{s,i}\bm{\delta}_i\rVert^2\ge0$, a descent step with $\widetilde{\bm{g}}_{\mathrm J,i}$ pulls the smashed data back toward $\bm{s}_{0,i}$, and more strongly for a displacement that changes the server output more. In this sense, CoeF-J reduces the error in \eqref{eq:corr} from the side of the displacement by keeping it small in the directions to which the server output is sensitive, whereas CoeF-D reduces the error from the side of the curvature. We note that this effect is not introduced by an additional regularizer, since it follows from the upper bound in \eqref{eq:major} with the coefficient $M$ fixed by the loss. We provide the full algorithms of CoeF-D and CoeF-J in Appendix~\ref{app:algorithm}.


\section{Theoretical Analysis}
\label{sec:theory}
We analyze why the gradient outer product can serve as a curvature surrogate
for CoeF-D and when its correction reduces the feedback error. For CoeF-J, we
examine when its sensitivity-based surrogate upper-bounds the loss. Throughout
this section, the server-side model is fixed at its round-initial value
$\theta^s$ as in \eqref{eq:identity}, and $\phi=f_s(\cdot;\theta^s)$ denotes the
server output of length $C$. The quantities $J_{s,i}$, $r_i$, $\Lambda_i$, and
$G_i$ are defined as in Section~\ref{sec:method}. Our analysis relies on the
following assumptions, whose resulting statements are proved in
Appendix~\ref{app:proofs}.
\begin{assumption}
\label{asm:lip}
$\nabla_s^2\mathcal L(\cdot;\theta^s)$ is $L_H$-Lipschitz on the segment from
$\bm s_{0,i}$ to $\bm s_i$.
\end{assumption}

\begin{assumption}
\label{asm:curv}
The Hessian of the loss with respect to the server output satisfies
$\nabla_\phi^2\mathcal L\preceq MI$ at every server output.
\end{assumption}

For softmax cross-entropy $\ell$, Assumption~\ref{asm:curv} holds with $M=\tfrac12$
\citep{bohning1992}; the output-space curvature of each sample in a mean-reduced
mini-batch of size $b$ is further scaled by $w=1/b$. Let
$H_{0,i}:=\nabla_s^2\mathcal L(\bm s_{0,i};\theta^s)$. By \eqref{eq:identity}
and Assumption~\ref{asm:lip}, the error of the compensated gradient satisfies
\begin{equation}
\label{eq:three}
\lVert\widetilde{\bm g}_i-\bm g_i\rVert
\le\lVert\widehat B_i-H_{0,i}\rVert\,\lVert\bm\delta_i\rVert
+\frac{L_H}{2}\lVert\bm\delta_i\rVert^2 .
\end{equation}
The first term measures the curvature mismatch at the transmitted smashed data;
the second accounts for the change of curvature along the displacement.

\vspace{-3mm}
\paragraph{CoeF-D.}
Let $D_i:=\operatorname{diag}(\bm g_{0,i}\odot\bm g_{0,i})$, so that
$\widehat B_{\mathrm D,i}=\lambda D_i$ in \eqref{eq:coefd}.
Proposition~\ref{prop1} gives both the curvature interpretation of this
diagonal outer product and the condition under which it improves the feedback
at a fixed displacement.

\begin{proposition}[Curvature Surrogate of CoeF-D]
\label{prop1}
\begin{enumerate}
\renewcommand{\labelenumi}{(\roman{enumi})}
\setlength{\itemsep}{0pt}
\item For the cross-entropy loss averaged over a mini-batch, suppose that its
labels are drawn independently from their predictive distributions at the
round-initial output. Then
\[
\mathbb E[\widehat B_{\mathrm D,i}]=\lambda w\operatorname{diag}(G_i).
\]
\item For a fixed displacement with $D_i\bm\delta_i\neq\bm 0$, let
\[
\lambda_i^\star:=
\frac{\langle D_i\bm\delta_i,\bar H_i\bm\delta_i\rangle}
{\lVert D_i\bm\delta_i\rVert^2}.
\]
For $\lambda>0$, CoeF-D has a smaller feedback error than gradient reuse if
and only if $\lambda<2\lambda_i^\star$.
\end{enumerate}
\end{proposition}

Part (i) relates the diagonal outer product to the GGN diagonal in the stated
predictive-label expectation; it does not assert equality for the observed
labels. Part (ii) compares the two feedback errors at the same displacement.
The threshold depends on how the CoeF-D correction aligns with the gradient
change needed at that displacement.

\vspace{-3mm}
\paragraph{CoeF-J.}
CoeF-J uses the sensitivity matrix $P_i=J_{s,i}^{\top}J_{s,i}$ in
\eqref{eq:major}. Assumption~\ref{asm:curv} bounds the loss change caused by a
change $\Delta$ in the server output by its first-order term
$\langle r_i,\Delta\rangle$ plus $M\lVert\Delta\rVert^2/2$. When the server
output changes linearly along $\bm\delta_i$, $\Delta=J_{s,i}\bm\delta_i$.
Together with $\bm g_{0,i}=J_{s,i}^{\top}r_i$, this gives the surrogate
$\widehat F_{\mathrm J}$ of \eqref{eq:major}.

\begin{proposition}[Loss Upper Bound]
\label{prop2}
Under Assumption~\ref{asm:curv}, consider the exact sensitivity matrix $P_i$
and the constant coefficient $M$ in \eqref{eq:major}. If the server output is
affine on the segment $[\bm s_{0,i},\bm s_{0,i}+\bm\delta_i]$, then
\[
\mathcal L(\bm s_{0,i}+\bm\delta_i;\theta^s)\le\widehat F_{\mathrm J}(\bm\delta_i),
\]
with equality at $\bm\delta_i=\bm 0$. If instead
$\lVert J_{s,i}\rVert\le B_\phi$ and $\lVert\nabla_s^2\phi_k\rVert\le L_\phi$ on
the segment for every output component $k$, then
\[
\mathcal L(\bm s_{0,i}+\bm\delta_i;\theta^s)
\le\widehat F_{\mathrm J}(\bm\delta_i)
+O\bigl(L_\phi\lVert\bm\delta_i\rVert^2\bigr)
  \quad\text{as }\bm\delta_i\to\bm 0,
\]
the remainder vanishes when $L_\phi=0$.
\end{proposition}

Under the affine condition, lowering $\widehat F_{\mathrm J}$ below its value at
$\bm\delta_i=\bm 0$ also keeps the loss below its value at the transmitted
smashed data. For a nonlinear server output, the same comparison includes the
stated remainder. The proposition concerns the exact $P_i$ and constant $M$; the
implementation instead uses a single random projection $\widehat P_i$ and
replaces $M$ by a per-sample coefficient $M_i\le wM$ computed from the predicted
output change $(\bm\delta_i^\top\widehat P_i\bm\delta_i)^{1/2}$.

\newlength{\coefw}\settowidth{\coefw}{\footnotesize CoeF-D~}
\newcommand{\oursrow}[1]{\makebox[\coefw][l]{CoeF-#1~}(ours)}
\providecommand{\std}[1]{\,{\tiny$\pm$}\,\tiny #1}
\setlength{\floatsep}{6pt plus 2pt minus 2pt}\setlength{\textfloatsep}{8pt plus 2pt minus 2pt}\setlength{\intextsep}{8pt plus 2pt minus 2pt}
\renewcommand{\topfraction}{0.98}\renewcommand{\textfraction}{0.02}\renewcommand{\floatpagefraction}{0.95}
\section{Experiments}
\label{sec:exp}

\paragraph{Models and datasets.} Vision experiments use ViT-Tiny and ViT-Base \citep{dosovitskiy2021vit} on CIFAR-100 \citep{krizhevsky2009cifar}, Food-101 \citep{bossard2014food101}, and Tiny-ImageNet \citep{le2015tinyimagenet}; language experiments use DistilRoBERTa and RoBERTa \citep{liu2019roberta} on 20 Newsgroups \citep{lang1995newsweeder}, MNLI \citep{williams2018mnli}, and SQuAD~v2.0 \citep{rajpurkar2018squad2}. All models are fine-tuned with LoRA adapters \citep{hu2022lora} of rank $4$ on the attention and MLP projections. Data are partitioned over $K=50$ clients under an IID split and a Dirichlet split (non-IID) with $\alpha=0.1$ on the vision tasks and $\alpha=0.5$ on the language tasks. SQuAD~v2.0 is partitioned IID only, since the extractive-QA loader has no label to skew. The cut layer, which determines the number of transformer blocks on the client side, is set to block 4 for the ViT models and RoBERTa and to block 2 for DistilRoBERTa, so that the numbers of trainable parameters of the client-side and server-side models are in a ratio of $1{:}2$. Appendix~\ref{app:cut} varies the cut layer for extended results. Accuracy is reported for classification and exact match (EM) and F1 for SQuAD~v2.0.

\vspace{-3mm}
\paragraph{Baselines.} Vanilla SFL exchanges smashed data and gradients at every batch. Gradient reuse (w/ perfect comp.) is oracle and Gradient reuse (w/o comp.) is naive baseline. AccSFL \citep{han2022localloss} and FSL-SAGE are the representatives of auxiliary-network-based SFL as we have discussed in Section \ref{sec:related_work}. While AccSFL relatively lacks of server feedback, FSL-SAGE periodically aligns the auxiliary network with server-side model using the collected activation history. GAS \citep{yang2024gas} and MU-SplitFed \citep{liang2025straggler} are imbalance-update based SFL addressing the straggler issue. For the implementation of GAS and MU-SplitFed, we adopts the straggler setting, but allowing per-batch communication, as we have discussed in Section \ref{sec:related_work}. All experiments, including the baselines, follow SplitFed~V1 \citep{thapa2022splitfed}, in which the server keeps one server-side model per client and aggregates them by FedAVG \citep{mcmahan2017fedavg} at the end of the round.

\vspace{-3mm}
\paragraph{Hyperparameters.} All methods are trained for $T=100$ rounds with $10\%$ client participation per round, one local epoch, a batch size of 32, and AdamW on both the client and server sides, after a warm-up of three rounds of vanilla SFL. For each method, the learning rate is selected from $\{10^{-4}, 5\times10^{-4}, 10^{-3}\}$. A run whose final accuracy remains at $1/C$ on a $C$-class task (1.0\% on CIFAR-100 and 0.5\% on Tiny-ImageNet) is reported as is and regarded as failing to train. CoeF-D uses a constant compensation factor $\lambda=3000$ with the cross-entropy loss averaged over the mini-batch ($w=1/32$), which is about $94$ times the reference scale $\lambda=1/w$ of Proposition~\ref{prop1}(i). Since the threshold in Proposition~\ref{prop1}(ii) depends on the curvature along the displacement rather than on $w$, we treat $\lambda$ as a hyperparameter and select it empirically. CoeF-J uses only the projected sensitivity matrix $\widehat P_i$ with a single random projection ($R=1$) on every task, including those with few classes, and replaces $M$ by a per-sample coefficient $M_i\le wM$ computed from the predicted output change $(\bm\delta_i^\top\widehat P_i\bm\delta_i)^{1/2}$; Proposition~\ref{prop2}, which assumes the exact $P_i$ and the constant $M$, does not cover these runs. We report the mean and standard deviation over three seeds, and the entries without standard deviation are single runs. In all tables, the best and second-best results among the methods other than the reference methods are shown in bold and underlined, respectively, in each column.

\begin{table}[t]
\centering
\scriptsize
\setlength{\tabcolsep}{2pt}
\renewcommand{\arraystretch}{1.05}
\newcolumntype{Y}{>{\centering\arraybackslash}X}
\caption{Comparison results: test accuracy (\%) on vision and language tasks, and EM and F1 (\%) on SQuAD v2.0. -- denotes that the method is not applicable to the task.}
\label{tab:main_iid}
\begin{tabularx}{\linewidth}{p{1.55cm} p{3.35cm} Y Y Y}
\toprule
\textbf{Model} & \textbf{Method} & \textbf{CIFAR-100} & \textbf{Food-101} & \textbf{Tiny-ImageNet} \\
\midrule
 \multirow{9}{*}{ViT-Tiny} & Vanilla SFL & 85.0\std{0.3} & 83.8\std{0.1} & 78.6\std{0.2} \\
    & Gradient reuse (w/ perfect comp.) & 84.5\std{0.3} & 83.3\std{0.0} & 78.3\std{0.3} \\
    & Gradient reuse (w/o comp.) & 78.4\std{0.4} & 75.3\std{2.1} & 72.8\std{0.2} \\
\cmidrule(l){2-5}
    & AccSFL & 65.6\std{1.2} & 67.5\std{0.1} & 47.4\std{0.2} \\
    & FSL-SAGE & 66.9\std{7.2} & 60.5\std{0.9} & 50.9\std{9.0} \\
    & GAS & 75.2\std{0.9} & 73.8\std{0.6} & 71.7\std{0.2} \\
    & MU-SplitFed & 1.0\std{0.0} & 1.0\std{0.0} & 0.5\std{0.0} \\
    & CoeF-D (ours) & \textbf{84.5\std{0.2}} & \underline{83.4\std{0.2}} & \textbf{78.3\std{0.1}} \\
    & CoeF-J (ours) & \underline{83.9\std{0.3}} & \textbf{83.6\std{0.0}} & \underline{77.8\std{0.1}} \\
\midrule
 \multirow{9}{*}{ViT-Base} & Vanilla SFL & 91.9\std{0.0} & 90.3\std{0.1} & 90.1\std{0.1} \\
    & Gradient reuse (w/ perfect comp.) & 91.7\std{0.1} & 90.0\std{0.2} & 90.0\std{0.1} \\
    & Gradient reuse (w/o comp.) & 91.3\std{0.2} & 88.5\std{0.2} & 89.5\std{0.3} \\
\cmidrule(l){2-5}
    & AccSFL & 89.0\std{0.3} & 87.7\std{0.2} & 84.1\std{0.2} \\
    & FSL-SAGE & 77.3\std{2.6} & 73.3\std{3.1} & 58.6\std{8.5} \\
    & GAS & 90.9\std{0.2} & 88.4\std{0.1} & 89.2\std{0.2} \\
    & MU-SplitFed & 59.1\std{0.5} & 65.4\std{0.3} & 68.0\std{0.2} \\
    & CoeF-D (ours) & \textbf{91.7\std{0.1}} & \textbf{90.1\std{0.1}} & \textbf{90.0\std{0.1}} \\
    & CoeF-J (ours) & \underline{91.4\std{0.0}} & \textbf{90.1\std{0.1}} & \underline{89.9\std{0.1}} \\
\end{tabularx}
\begin{tabularx}{\linewidth}{p{1.55cm} p{3.35cm} Y Y Y Y Y}
\midrule
\multirow{2}{*}{\textbf{Model}} & \multirow{2}{*}{\textbf{Method}} & \multirow{2}{*}{\textbf{20 Newsgroups}} & \multirow{2}{*}{\textbf{MNLI-m}} & \multirow{2}{*}{\textbf{MNLI-mm}} & \multicolumn{2}{c}{\textbf{SQuAD v2.0}} \\
\cmidrule(l){6-7}
 & & & & & \textbf{EM} & \textbf{F1} \\
\midrule
 \multirow{9}{*}{DistilRoBERTa} & Vanilla SFL & 69.8\std{0.2} & 83.3\std{0.3} & 83.7\std{0.3} & 71.3\std{0.2} & 74.5\std{0.2} \\
    & Gradient reuse (w/ perfect comp.) & 69.3\std{0.2} & 83.0\std{0.1} & 83.1\std{0.2} & 70.2\std{2.1} & 73.1\std{1.5} \\
    & Gradient reuse (w/o comp.) & 69.3\std{0.0} & 37.1\std{6.2} & 37.7\std{7.1} & 49.9\std{0.2} & 50.0\std{0.2} \\
\cmidrule(l){2-7}
    & AccSFL & 66.6\std{0.3} & 77.7\std{0.3} & 78.5\std{0.3} & 61.6\std{1.2} & 64.7\std{1.1} \\
    & FSL-SAGE & 62.0\std{0.8} & 78.8\std{0.9} & 79.7\std{0.7} & 62.1\std{0.8} & 65.2\std{0.7} \\
    & GAS & 68.0\std{0.5} & 38.9\std{2.5} & 39.8\std{3.1} & -- & -- \\
    & MU-SplitFed & 5.1\std{0.1} & 33.6\std{1.3} & 33.7\std{1.1} & 10.7\std{0.0} & 11.5\std{0.0} \\
    & CoeF-D (ours) & \underline{69.7\std{0.1}} & \underline{79.7\std{0.5}} & \underline{80.0\std{0.3}} & \underline{65.7\std{5.3}} & \underline{68.2\std{4.4}} \\
    & CoeF-J (ours) & \textbf{69.8\std{0.3}} & \textbf{83.0\std{0.4}} & \textbf{83.4\std{0.2}} & \textbf{68.6\std{1.1}} & \textbf{71.7\std{0.9}} \\
\midrule
 \multirow{9}{*}{RoBERTa} & Vanilla SFL & 70.4\std{0.5} & 87.3\std{0.2} & 87.1\std{0.0} & 79.3\std{0.4} & 82.4\std{0.3} \\
    & Gradient reuse (w/ perfect comp.) & 70.0\std{0.4} & 86.5\std{0.5} & 86.4\std{0.2} & 77.9\std{2.0} & 81.0\std{1.6} \\
    & Gradient reuse (w/o comp.) & 69.5\std{0.6} & 32.4\std{0.4} & 32.6\std{0.5} & 47.0\std{3.9} & 47.4\std{3.6} \\
\cmidrule(l){2-7}
    & AccSFL & 67.0\std{0.7} & 84.3\std{0.3} & 84.9\std{0.2} & 71.7\std{0.5} & 74.7\std{0.5} \\
    & FSL-SAGE & 64.5\std{0.3} & 83.2\std{0.5} & 83.6\std{0.6} & 68.6\std{0.8} & 71.9\std{0.8} \\
    & GAS & 68.9\std{0.9} & 35.5\std{4.9} & 36.1\std{6.1} & -- & -- \\
    & MU-SplitFed & 5.3\std{0.0} & 34.5\std{1.3} & 34.5\std{1.1} & 10.7\std{0.0} & 11.5\std{0.0} \\
    & CoeF-D (ours) & \underline{70.1\std{0.5}} & \textbf{87.2\std{0.1}} & \underline{86.8\std{0.2}} & \underline{77.2\std{0.6}} & \underline{79.4\std{0.7}} \\
    & CoeF-J (ours) & \textbf{70.6\std{0.3}} & \textbf{87.2\std{0.1}} & \textbf{86.9\std{0.0}} & \textbf{78.2\std{0.3}} & \textbf{81.4\std{0.3}} \\
\bottomrule
\end{tabularx}
\vspace{-8pt}
\end{table}

\vspace{-3mm}
\paragraph{Comparison results.} Table~\ref{tab:main_iid} reports the results under IID split. Although CoeF-SFL communicates only once per round, its better variant stays within 0.5\,pp of vanilla SFL on every vision task, within 0.3\,pp on 20 Newsgroups and MNLI, and within 2.8\,pp on SQuAD v2.0. CoeF-J is particularly effective with DistilRoBERTa, where it exceeds CoeF-D by 3.3--3.4\,pp on MNLI and by 2.9\,pp in EM on SQuAD v2.0 with a much smaller standard deviation. In several columns, CoeF-SFL even slightly exceeds the oracle, e.g., 87.2\% versus 86.5\% on MNLI-m with RoBERTa, which may be attributed to the regularization effect of the compensation term.

Gradient reuse (w/o comp.) collapses to near-chance accuracy on MNLI and to around 50\% on SQuAD v2.0, and CoeF-SFL recovers the performance by up to 54.8\,pp. Among the methods other than the reference methods, CoeF-SFL achieves the best performance in every column, improving on the best auxiliary-network-based method by 16.1--27.4\,pp with ViT-Tiny, 2.4--5.9\,pp with ViT-Base, and 2.0--6.7\,pp on the language tasks.

The auxiliary-network-based methods degrade more on the vision tasks than on the language tasks. With ViT-Tiny, they fall 16.3--31.2\,pp below vanilla SFL, whereas their gaps remain within 10.7\,pp on all the language tasks. We conjecture that this difference stems from how strongly the intermediate representations are compressed. Vision representations progressively discard label-irrelevant input information \citep{shwartz2017opening, wang2021infopro}, and a local objective at a shallow cut layer accelerates this compression before the server-side model can exploit the information. In contrast, masked language models preserve more information about the input tokens \citep{voita2019bottom}, so a local objective is less likely to collapse what the subsequent layers require. Consistently, the largest gap among the language tasks appears on SQuAD v2.0, which requires fine-grained token-level information, and with ViT-Base, whose client-side model has a larger capacity, the gap of AccSFL shrinks to 2.6--6.0\,pp, in line with \citep{wang2021infopro}. FSL-SAGE, however, still falls 14.6--31.5\,pp behind with ViT-Base and varies by up to 8.5\,pp across seeds, which suggests unstable training in our setting.

GAS is competitive with ViT-Base and on 20 Newsgroups, staying within 1.9\,pp of vanilla SFL, but it falls 6.9--10.0\,pp behind with ViT-Tiny and fails to train on MNLI. It is not applicable to SQuAD v2.0, since its activation generation is conditioned on class labels, which the span-extraction task does not provide. MU-SplitFed yields near-trivial predictions on most tasks, e.g., $1/C$ with ViT-Tiny, on 20 Newsgroups, and on MNLI. We attribute this to the scalar feedback from the server, which conveys too little information for the client-side model to learn without backpropagation under the same number of rounds.

\vspace{-3mm}
\paragraph{Extended results and discussion.}
In Appendix~\ref{app:extended}, we provide additional experiments across different cut layers and under the non-IID split. CoeF-SFL consistently outperforms the baselines regardless of the cut layer, and its advantage becomes more pronounced under the non-IID split. In Appendix~\ref{app:discussion}, we compare the gradient outer product \citep{zheng2017dcasgd} adopted by CoeF-D with two other Hessian approximations, i.e., a low-rank approximation of the Gauss--Newton matrix \citep{schraudolph2002ggn,martens2020natural} computed with the Lanczos method \citep{lanczos1950iteration} denoted by GGN, and a diagonal approximation estimated with the Hutchinson estimator \citep{bekas2007diag}, denoted by Diagonal. The gradient outer product performs best among them in our setting. In Appendix~\ref{app:cost}, we further analyze the communication and computation costs of the methods, which reveals the trade-offs among them beyond the performance and shows that CoeF-SFL achieves the highest accuracy without increasing the burden on the clients.

\FloatBarrier


\section{Conclusion}
\label{sec:conclusion}
In this paper, we identified that auxiliary-network-based SFL reduces communication by replacing the E2E objective with a local objective, which limits the collaborative training between the client and the server. To retain the E2E objective under once-per-round communication, we proposed CoeF-SFL, in which the client reuses the gradients received at the beginning of each round and compensates their staleness with a curvature-based correction in the activation space. CoeF-D realizes this correction with a diagonal gradient outer product that requires no additional server computation, and CoeF-J exploits the Jacobian-based Hessian of a surrogate loss that upper-bounds the true loss. Across vision and language tasks, model capacities, cut layers, and data distributions, CoeF-SFL achieves performance close to that of vanilla SFL and consistently outperforms the auxiliary-network-based methods under the same communication frequency. Extending the analysis to the joint dynamics of the client-side and server-side updates remains an interesting direction for future work.

\bibliographystyle{iclr2027_conference}
\bibliography{iclr2027_conference,coef_sfl_refs_20260917}

\appendix
\newpage
\section{Algorithm of CoeF-SFL}
\label{app:algorithm}

\begin{algorithm}[ht]
\caption{CoeF-SFL for a sampled client in one round. Lines marked (D) and (J) are executed only by CoeF-D and CoeF-J, respectively, and the quantities in (J) lines are computed for each sample in the mini-batch.}
\label{alg:coef}
\small
\begin{algorithmic}[1]
\REQUIRE client-side model $\theta^{c}$ and server-side model $\theta^{s}$ at the start of the round, $I$ local updates, step size $\eta$; (D) compensation factor $\lambda$; (J) curvature bound $M$ and number $R$ of random projections
\STATE \textbf{Round-initial exchange}
\STATE \phantom{dd} \textbf{Client}: compute $\bm{s}=\left\{f_c\bigl(x;\theta^{c}\bigr)\right\}$ from its local data $\mathcal{D}=\left\{(x,y)\right\}$, and send $\bm{s}$ to the server together \phantom{dd} with the labels $\{\bm{y}\}$
\STATE \phantom{dd} \textbf{Server}: computes the gradients $\bm{g}=\left\{\nabla_s\mathcal{L}(s;\theta^s)\right\}$
\STATE \quad \phantom{dd} (J) draw $\bm{v}_{1},\dots,\bm{v}_{R}\in\{-1,+1\}^{C}$ with independent random signs, and compute $\bm{u}_{r}=J_{s}^{\top}\bm{v}_{r}$
\STATE \phantom{dd} send $\bm{g}$ to the client, together with $\{\bm{u}_{r}\}_{r}$ for CoeF-J
\STATE \textbf{Local updates}, starting from $\theta^{c}_{0}=\theta^{c}$ and $\theta^{s}_{0}=\theta^{s}$
\FOR{$i=0,\dots,I-1$}
\STATE Client:  $\bm{\delta}_i\leftarrow\bm{s}_{i}-\bm{s}_{0,i}$
\STATE \quad (D) $\widetilde{\bm{g}}_i\leftarrow\bm{g}_{0,i}+\lambda\,\mathrm{diag}(\bm{g}_{0,i}\odot\bm{g}_{0,i})\bm{\delta}_i$ \hfill\eqref{eq:coefd}
\STATE \quad (J) $\widehat{P}_i\bm{\delta}_i\leftarrow R^{-1}\sum_{r=1}^{R}\bm{u}_{i,r}\langle\bm{u}_{i,r},\bm{\delta}_i\rangle$, \quad $\widetilde{\bm{g}}_i\leftarrow\bm{g}_{0,i}+M\,\widehat{P}_i\bm{\delta}_i$ \hfill\eqref{eq:coefj}
\STATE Client: $\theta^{c}_{i+1}\leftarrow\theta^{c}_{i}-\eta\,J_{i}^{\top}\widetilde{\bm{g}}_i$
\STATE Server: $\theta^{s}_{i+1}\leftarrow\theta^{s}_{i}-\eta\,\nabla_{\theta^{s}}\mathcal{L}(\bm{s}_{0,i};\theta^{s}_{i})$
\ENDFOR
\STATE Aggregate $\theta^{c}_{I}$ and $\theta^{s}_{I}$ over the sampled clients by FedAVG
\end{algorithmic}
\end{algorithm}

\section{Proofs of Theoretical Results}
\label{app:proofs}
Throughout this appendix, the server-side model is fixed at its round-initial value $\theta^s$. We write
$H(\bm s):=\nabla_s^2\mathcal L(\bm s;\theta^s)$ and
$H_{0,i}=H(\bm s_{0,i})$.
\subsection{CoeF-D}
\label{app:common}
\begin{lemma}[Feedback error]
\label{lem:feedback}
Let
$\bar H_i:=\int_0^1H(\bm s_{0,i}+\tau\bm\delta_i)\,d\tau$.
The identity \eqref{eq:identity} gives
$\bm g_i-\bm g_{0,i}=\bar H_i\bm\delta_i$.
Thus gradient reuse has error $-\bar H_i\bm\delta_i$,
while the compensated gradient has error
$(\widehat B_i-\bar H_i)\bm\delta_i$.
Under Assumption~\ref{asm:lip}, this error satisfies
\eqref{eq:three}.
\end{lemma}

\begin{proof}
Applying the fundamental theorem of calculus to
$\tau\mapsto\nabla_s\mathcal L
(\bm s_{0,i}+\tau\bm\delta_i;\theta^s)$
gives \eqref{eq:identity}.
Assumption~\ref{asm:lip} gives
\[
\|\bar H_i-H_{0,i}\|
\le\int_0^1L_H\tau\|\bm\delta_i\|\,d\tau
=\frac{L_H}{2}\|\bm\delta_i\|.
\]
Adding and subtracting $H_{0,i}$ in
$\widehat B_i-\bar H_i$ proves \eqref{eq:three}.
\end{proof}

\begin{proof}[Proof of Proposition~\ref{prop1}]
(i) For a sample in the mini-batch, let $p$ be the predictive distribution, $y$ its label, and $w=1/b$. Its output-gradient contribution is $r=w(p-\bm e_y)$, and its output-space Hessian is
$\Lambda=w(\operatorname{diag}(p)-pp^\top)$.
For $y\sim p$,
\[
\mathbb E[\bm e_y]=p,
\qquad
\mathbb E[\bm e_y\bm e_y^\top]=\operatorname{diag}(p),
\]
so
\[
\mathbb E[rr^\top]
=w^2(\operatorname{diag}(p)-pp^\top)
=w\Lambda.
\]
By the chain rule, the corresponding activation-gradient outer product has expectation $wJ_s^\top\Lambda J_s$.
For the full mini-batch, independent predictive labels make the cross-sample terms vanish because each centered gradient contribution has zero mean. Hence
$\mathbb E[\bm g_{0,i}\bm g_{0,i}^{\top}]=wG_i$.
Taking diagonals and using
$\widehat B_{\mathrm D,i}
=\lambda\operatorname{diag}(\bm g_{0,i}\odot\bm g_{0,i})$
gives the claim.
(ii) At the fixed displacement in the proposition,
\[
\begin{aligned}
\lVert \widetilde{\bm g}_{\mathrm D,i}-\bm g_i\rVert ^2
-\lVert \bm g_{0,i}-\bm g_i\rVert ^2
&=
\lVert \lambda D_i\bm\delta_i-\bar H_i\bm\delta_i\rVert ^2
-\lVert \bar H_i\bm\delta_i\rVert ^2\\
&=\lambda^2\lVert D_i\bm\delta_i\rVert ^2
-2\lambda\langle D_i\bm\delta_i,\bar H_i\bm\delta_i\rangle\\
&=\lambda\lVert D_i\bm\delta_i\rVert ^2
(\lambda-2\lambda_i^\star).
\end{aligned}
\]
Since $\lambda>0$ and $D_i\bm\delta_i\neq\bm0$, this difference is negative if and only if $\lambda<2\lambda_i^\star$.
\end{proof}
\subsection{CoeF-J}
\label{app:coefj}
For completeness, the softmax cross-entropy loss satisfies the curvature bound used by Proposition~\ref{prop2}. For $p=\operatorname{softmax}(\phi)$,
\[
\nabla_\phi^2\ell
=\operatorname{diag}(p)-pp^\top\preceq\tfrac12 I.
\]
Indeed, for any $\bm v$, its quadratic form is the variance of $v_K$ for $K\sim p$. This variance is at most one quarter of the squared range of the entries of $\bm v$, and that squared range is at most $2\lVert \bm v\rVert ^2$. Mean reduction multiplies each sample's curvature by $w=1/b$; hence $M=\tfrac12$ is valid in Assumption~\ref{asm:curv}.
\begin{proof}[Proof of Proposition~\ref{prop2}]
Write $\mathcal L_\phi$ for the loss as a function of the server output and put $\bar\phi=\phi(\bm s_{0,i})$. Assumption~\ref{asm:curv} and Taylor's formula give, for any output change $\Delta$,
\[
\mathcal L_\phi(\bar\phi+\Delta)
\le
\mathcal L_\phi(\bar\phi)
+\langle r_i,\Delta\rangle
+\frac M2\lVert \Delta\rVert ^2.
\]
If the server output is affine on the segment, its actual change is
$\Delta=J_{s,i}\bm\delta_i$. Since
$\bm g_{0,i}=J_{s,i}^{\top}r_i$ and
$\lVert J_{s,i}\bm\delta_i\rVert ^2
=\bm\delta_i^\top P_i\bm\delta_i$,
the right-hand side is exactly
$\widehat F_{\mathrm J}(\bm\delta_i)$.
At $\bm\delta_i=\bm0$, both sides equal
$\mathcal L(\bm s_{0,i};\theta^s)$.
For a nonlinear server output, write
\[
\phi(\bm s_{0,i}+\bm\delta_i)
=\bar\phi+J_{s,i}\bm\delta_i+\bm\varrho.
\]
Taylor's integral formula and the stated bounds on the output components give
$\lVert \bm\varrho\rVert \le
(\sqrt C L_\phi/2)\lVert \bm\delta_i\rVert ^2$.
Using the actual output change in the preceding loss bound instead of
$J_{s,i}\bm\delta_i$ adds
\[
\langle r_i,\bm\varrho\rangle
+M\langle J_{s,i}\bm\delta_i,\bm\varrho\rangle
+\frac M2\lVert \bm\varrho\rVert ^2.
\]
Since
$\lVert J_{s,i}\bm\delta_i\rVert
\le B_\phi\lVert \bm\delta_i\rVert $,
this addition is at most
$\tfrac12\varepsilon(\lVert \bm\delta_i\rVert )
\lVert \bm\delta_i\rVert ^2$, where
\[
\varepsilon(u)
=\sqrt C L_\phi(\lVert r_i\rVert +MB_\phi u)
+\frac14CM L_\phi^2u^2.
\]
It vanishes when $L_\phi=0$, completing the proof.
\end{proof}

\newpage
\section{Additional Experiments}
\label{app:extended}
\subsection{Effect of the Cut Layer}
\label{app:cut}

\begin{table}[ht]
\centering
\scriptsize
\setlength{\tabcolsep}{1.5pt}
\renewcommand{\arraystretch}{1.05}
\caption{Effect of the cut layer on ViT-Tiny/CIFAR-100 (IID): test accuracy (\%). Cut $k$ places the first $k$ of the twelve transformer blocks on the client, and the last row reports the ratio of trainable parameters, excluding the embedding layer, between the client-side and server-side models.}
\label{tab:cut}
\providecommand{\stdc}[1]{\,\mbox{\scalebox{0.85}{\tiny$\pm$\,#1}}}
\begin{tabular}{l c c c c c}
\toprule
\textbf{Method} & \textbf{cut 1} & \textbf{cut 2} & \textbf{cut 4} & \textbf{cut 6} & \textbf{cut 8} \\
\midrule
Vanilla SFL & 85.1\stdc{0.1} & 84.9\stdc{0.1} & 85.0\stdc{0.3} & 84.9\stdc{0.3} & 84.8\stdc{0.3} \\
Gradient reuse (w/ perfect comp.) & 84.8\stdc{0.1} & 84.8\stdc{0.1} & 84.5\stdc{0.3} & 84.3\stdc{0.2} & 83.9\stdc{0.2} \\
Gradient reuse (w/o comp.) & 84.5\stdc{0.3} & 84.3\stdc{0.2} & 78.4\stdc{0.4} & 78.5\stdc{0.3} & 77.7\stdc{0.3} \\
\cmidrule(l){1-6}
AccSFL & 71.3\stdc{1.4} & 62.0\stdc{1.3} & 65.6\stdc{1.2} & 71.4\stdc{0.4} & 76.0\stdc{0.2} \\
FSL-SAGE & 47.0\stdc{17.2} & 55.2\stdc{7.7} & 66.9\stdc{7.2} & 69.6\stdc{5.7} & 73.7\stdc{4.1} \\
GAS & 82.9\stdc{0.2} & 83.1\stdc{0.3} & 75.2\stdc{0.9} & 75.6\stdc{0.6} & 76.1\stdc{0.7} \\
MU-SplitFed & 1.0\stdc{0.0} & 1.0\stdc{0.0} & 1.0\stdc{0.0} & 1.0\stdc{0.0} & 9.5\stdc{0.2} \\
CoeF-D (ours) & \underline{84.9\stdc{0.2}} & \textbf{85.0\stdc{0.2}} & \textbf{84.5\stdc{0.2}} & \textbf{84.3\stdc{0.2}} & \textbf{83.8\stdc{0.3}} \\
CoeF-J (ours) & \textbf{85.0\stdc{0.1}} & \underline{84.8\stdc{0.1}} & \underline{83.9\stdc{0.3}} & \underline{82.5\stdc{0.1}} & \underline{81.3\stdc{0.1}} \\
\midrule
Client : server & 1\,:\,11.04 & 1\,:\,5.02 & 1\,:\,2.01 & 1.01\,:\,1 & 1.98\,:\,1 \\
\bottomrule
\end{tabular}
\vspace{-8pt}
\end{table}

\begin{table}[ht]
\centering
\scriptsize
\setlength{\tabcolsep}{1.5pt}
\renewcommand{\arraystretch}{1.05}
\caption{Effect of the cut layer on DistilRoBERTa/20 Newsgroups (IID): test accuracy (\%). Cut $k$ places the first $k$ of the six transformer blocks on the client, and the last row reports the ratio of trainable parameters, excluding the embedding layer, between the client-side and server-side models.}
\label{tab:cut_ng}
\providecommand{\stdc}[1]{\,\mbox{\scalebox{0.85}{\tiny$\pm$\,#1}}}
\begin{tabular}{l c c c c}
\toprule
\textbf{Method} & \textbf{cut 1} & \textbf{cut 2} & \textbf{cut 3} & \textbf{cut 4} \\
\midrule
Vanilla SFL & 69.6\stdc{0.2} & 69.8\stdc{0.2} & 69.7\stdc{0.3} & 69.7\stdc{0.3} \\
Gradient reuse (w/ perfect comp.) & 69.6\stdc{0.2} & 69.3\stdc{0.2} & 69.2\stdc{0.3} & 69.2\stdc{0.2} \\
Gradient reuse (w/o comp.) & 69.5\stdc{0.0} & 69.3\stdc{0.0} & 68.8\stdc{0.2} & 67.7\stdc{0.1} \\
\cmidrule(l){1-5}
AccSFL & 68.2\stdc{0.3} & 66.6\stdc{0.3} & 66.8\stdc{0.3} & 67.8\stdc{0.3} \\
FSL-SAGE & 59.9\stdc{0.7} & 62.0\stdc{0.8} & 64.8\stdc{0.4} & 66.5\stdc{0.0} \\
GAS & 68.4\stdc{0.5} & 68.0\stdc{0.5} & 66.9\stdc{0.2} & 66.6\stdc{0.7} \\
MU-SplitFed & 5.4\stdc{0.7} & 5.1\stdc{0.1} & 5.6\stdc{0.2} & 7.0\stdc{0.9} \\
CoeF-D (ours) & \underline{69.7\stdc{0.1}} & \underline{69.7\stdc{0.1}} & \underline{69.3\stdc{0.2}} & \underline{69.2\stdc{0.3}} \\
CoeF-J (ours) & \textbf{69.8\stdc{0.1}} & \textbf{69.8\stdc{0.3}} & \textbf{69.8\stdc{0.1}} & \textbf{69.5\stdc{0.3}} \\
\midrule
Client : server & 1\,:\,5.09 & 1\,:\,2.04 & 1\,:\,1.03 & 1.92\,:\,1 \\
\bottomrule
\end{tabular}
\end{table}

\subsection{Results under the Non-IID Split}
\label{app:noniid}

\begin{table}[ht]
\centering
\scriptsize
\setlength{\tabcolsep}{3pt}
\renewcommand{\arraystretch}{1.05}
\caption{ViT-Tiny/vision tasks (non-IID): test accuracy (\%).}
\label{tab:vision_noniid}
\begin{tabular}{c l c c c}
\toprule
\textbf{Model} & \textbf{Method} & \textbf{CIFAR-100} & \textbf{Food-101} & \textbf{Tiny-ImageNet} \\
\midrule
 \multirow{9}{*}{ViT-Tiny} & Vanilla SFL & 70.9\std{1.2} & 63.9\std{3.1} & 59.1\std{3.3} \\
    & Gradient reuse (with perfect comp.) & 66.9\std{0.5} & 55.6\std{1.5} & 54.7\std{2.9} \\
    & Gradient reuse (w/o comp.) & 1.0\std{0.1} & 1.0\std{0.0} & 0.5\std{0.0} \\
\cmidrule(l){2-5}
    & AccSFL & 28.2\std{1.3} & 32.5\std{4.3} & 18.8\std{2.1} \\
    & FSL-SAGE & 33.6\std{2.0} & 30.3\std{5.5} & 20.0\std{2.3} \\
    & GAS & 46.0\std{10.2} & 1.0\std{0.0} & 0.6\std{0.1} \\
    & MU-SplitFed & 1.0\std{0.0} & 1.0\std{0.0} & 0.5\std{0.0} \\
    & CoeF-D (ours) & \underline{67.3\std{1.5}} & \underline{53.5\std{2.4}} & \underline{54.9\std{2.7}} \\
    & CoeF-J (ours) & \textbf{69.1\std{1.2}} & \textbf{62.9\std{3.2}} & \textbf{57.7\std{3.4}} \\
\bottomrule
\end{tabular}
\vspace{-8pt}
\end{table}

\begin{table}[ht]
\centering
\scriptsize
\setlength{\tabcolsep}{2.5pt}
\renewcommand{\arraystretch}{1.05}
\caption{DistilRoBERTa/language tasks (non-IID): test accuracy (\%).}
\label{tab:lang_noniid}
\begin{tabular}{c l c c c}
\toprule
\textbf{Model} & \textbf{Method} & \textbf{20 Newsgroups} & \textbf{MNLI-m} & \textbf{MNLI-mm} \\
\midrule
 \multirow{9}{*}{DistilRoBERTa} & Vanilla SFL & 67.2\std{0.6} & 80.4\std{0.5} & 80.9\std{0.7} \\
    & Gradient reuse (with perfect comp.) & 64.9\std{1.3} & 74.6\std{4.2} & 74.5\std{4.8} \\
    & Gradient reuse (w/o comp.) & 62.5\std{1.7} & 33.3\std{1.5} & 33.3\std{1.4} \\
\cmidrule(l){2-5}
    & AccSFL & 61.1\std{1.5} & 71.4\std{1.7} & 71.9\std{2.2} \\
    & FSL-SAGE & 53.2\std{1.6} & 73.0\std{1.3} & 73.4\std{1.3} \\
    & GAS & 64.2\std{0.9} & 37.9\std{6.2} & 38.3\std{7.0} \\
    & MU-SplitFed & 5.2\std{0.1} & 34.5\std{1.3} & 34.5\std{1.1} \\
    & CoeF-D (ours) & \underline{65.7\std{1.0}} & \underline{75.7\std{0.7}} & \underline{75.6\std{0.9}} \\
    & CoeF-J (ours) & \textbf{67.1\std{0.6}} & \textbf{77.0\std{0.5}} & \textbf{77.4\std{0.9}} \\
\bottomrule
\end{tabular}
\vspace{-8pt}
\end{table}

\newpage
\section{Further Analysis}
\subsection{Comparison of Hessian Approximations}
\label{app:hessian}
Here, we compare the gradient outer product \citep{zheng2017dcasgd} adopted by CoeF-D with two other Hessian approximations, i.e., a low-rank approximation of the Gauss--Newton matrix \citep{schraudolph2002ggn,martens2020natural} computed with the Lanczos method \citep{lanczos1950iteration} denoted by GGN, and a diagonal approximation estimated with the Hutchinson estimator \citep{bekas2007diag}, denoted by Hutchinson diagonal. We explain each method and introduce the corresponding  first.

\label{app:discussion}
\begin{table}[ht]
\centering
\scriptsize
\setlength{\tabcolsep}{4pt}
\caption{Notation for further analysis.}
\label{tab:notation_cost}
\begin{tabular}{l l r}
\toprule
\textbf{Symbol} & \textbf{Meaning} &  \\
\midrule
$N$ & samples per client  \\
$B$ & mini-batches per client and round  \\
$d$ & dimension of the smashed data of one sample\\
$\mathrm f^{c}$, $\mathrm b^{c}$ & client forward / backward pass \\
$\mathrm f^{s}$, $\mathrm b^{s}$ & server forward / full backward pass \\
$\mathrm b^{s}_{\mathrm{in}}$ & server backward pass to the smashed data  \\
\midrule
$k$ & rank of the low-rank Fisher (GGN) \\
$C$ & number of classes, the dimension of the server output \\
$\bm v$ & input vector of a Gauss--Newton matrix-vector product \\
$J_s$ & Jacobian of the server output with respect to the smashed data of one sample ($C\times d$) \\
$\Lambda_i$ & Hessian of the loss with respect to the server output at the round-initial output of sample $i$ \\ 
$J_s^{\top}\Lambda J_s$ & generalized Gauss--Newton matrix of one sample  \\
$U$ & eigenvectors for the $k$ leading eigenvalues of $J_s^{\top}\Lambda J_s$  \\
$\Theta$ & diagonal matrix of the $k$ leading eigenvalues ($k\times k$) \\
$m$ & Lanczos iterations per sample (low-rank Fisher)  \\
$m_d$ & Rademacher probes per sample (Hutchinson diagonal) \\
$R$ & random projections per sample (CoeF-J)  \\
$\widehat{B}$ & surrogate hessian per sample (GGN, CoeF-D)\\
$\lambda$ & compensation factor (CoeF-D) \\
\midrule
$N_{\mathrm{buf}}$ & samples stored in the FSL-SAGE buffer \\
$E$ & epochs of auxiliary-network training per alignment  \\
$l$ & rounds between alignments (FSL-SAGE)  \\
$\mathrm a$ & training cost of the auxiliary network \\
$P$ & perturbations per update (MU-SplitFed) \\
$\tau$ & server updates per client update (MU-SplitFed)  \\
$N_g$ & activations generated by the server per client and round (GAS) \\
\bottomrule
\end{tabular}
\end{table}

\paragraph{GGN.} The server runs $m$ Lanczos iterations~\citep{lanczos1950iteration} on the per-sample generalized Gauss--Newton operator $\bm v\mapsto G_i\bm v=J_{s,i}^{\top}\Lambda_iJ_{s,i}\bm v$, evaluated through Gauss--Newton matrix-vector products without forming the matrix~\citep{schraudolph2002ggn}, and keeps the $k$ leading eigenpairs $(U,\Theta)$. For the cross-entropy loss this operator coincides with the Fisher information~\citep{martens2020natural}. The client applies $\widehat B_i\bm\delta_i=U\Theta U^{\top}\bm\delta_i$ at a cost of $k(\mathrm{f}_s+ \mathrm{b}_s)$ per sample. The communication cost for downlink grows to $(1+k)d$ values per sample. With $m=8$ and $k=4$, the server needs $41.45$ GFLOPs per sample, of which the Lanczos iterations account for $34.23$ GFLOPs. 

\paragraph{Hutchinson diagonal.} The server estimates $\operatorname{diag}(H_{0,i})$ by $m_d^{-1}\sum_{j=1}^{m_d}\bm z_j\odot H_{0,i}\bm z_j$ with Rademacher vectors $\bm z_j$~\citep{hutchinson1990trace,bekas2007diag}. Each product $H_{0,i}\bm z_j$ is obtained by a central finite difference of two server gradient queries at $\bm s_{0,i}\pm\epsilon\bm z_j$, so the server adds $2m_d(\mathrm f^s+\mathrm b^s_{\mathrm{in}})$ per sample. With $m_d=8$ this gives $7.226+16\times3.591=64.69$ GFLOPs per sample. The diagonal has the size of $\bm g_{0,i}$, so the downlink doubles, and the client correction is elementwise.

\paragraph{CoeF-D (diagonal outer product).} $\widehat B_{\mathrm D}=\lambda\operatorname{diag}(\bm g_{0,i}\odot\bm g_{0,i})$ is formed by the client from $\bm g_{0,i}$, so the server cost and the downlink equal those of gradient reuse, and the client adds $O(d)$ per sample.

\paragraph{Choice of curvature approximation.}
Table~\ref{tab:approx} compares CoeF-D with two alternatives, which are a low-rank Gauss--Newton approximation computed with the Lanczos method (GGN) and a stochastic estimate of its diagonal (Diagonal). Under IID data, GGN matches CoeF-D but increases the server computation by $5.7\times$ and the communication cost by $5.0\times$, while Diagonal costs $8.9\times$ more computation and barely improves over gradient reuse without compensation. Under non-IID data, both alternatives collapse, and Diagonal falls to the chance level at cut 4. CoeF-D instead retains 69.1\% at cut 4, since its correction scales with $\bm{g}_{0,i}\odot\bm{g}_{0,i}$ and thus stays small for samples that are already well fitted. As CoeF-D achieves this with the same computation and communication as gradient reuse without compensation, we adopt the diagonal gradient outer product as the default compensation.

\begin{table}[t]
\caption{Gradient reuse with different curvature approximations on ViT-Tiny/CIFAR-100. The top table reports test accuracy (\%), and the bottom table reports the server computation per sample and the communication cost at cut 4. The values are measured in our setting from Section \ref{sec:exp}}.
\label{tab:approx}
\centering\scriptsize\setlength{\tabcolsep}{4pt}\renewcommand{\arraystretch}{1.05}
\begin{tabular}{l ll ll ll}
\toprule
\multirow{2}{*}[-0.6ex]{Method} & \multicolumn{2}{l}{cut 1} & \multicolumn{2}{l}{cut 4} & \multicolumn{2}{l}{cut 8}  \\
\cmidrule(lr){2-3} \cmidrule(lr){4-5} \cmidrule(lr){6-7}
 & IID & non-IID & IID & non-IID & IID & non-IID \\
\midrule
GGN & 84.8 & 46.3 & 84.5 & 19.4 & 83.0 & 7.8   \\
Diagonal & 84.7 & 66.2 & 78.7 & 1.0 & 78.0 & 7.7   \\
\cmidrule(l){1-7}
CoeF-D & 84.8 & 72.4 & 84.5 & 69.1 & 83.7 & 48.8   \\
\bottomrule
\end{tabular}

\vspace{4pt}
\begin{tabular}{l c c c c}
\toprule
\textbf{Method} & \textbf{Server computation (GFLOPs)} & \textbf{Computation complexity} & \textbf{Comm. cost (MB)} & \textbf{Comm. volume} \\
\midrule
Gradient reuse (w/o comp.) & 7.23 & $2(\mathrm{f}^s + \mathrm{b}^s)$ & 144.29 & $O(Nd)$ \\
\cmidrule(l){1-5}
GGN & 41.45 & $(k+1)(\mathrm{f}^{s}{+}\mathrm{b}^{s})$ & 721.45 & $O((k+1)Nd)$ \\
Hutchinson diagonal & 64.69 & $ \mathrm{f}^s + \mathrm{b}^s +2m_d(\mathrm{f}^s +\mathrm{b} ^s _{\mathrm{in}})$ & 288.58 & $O(2Nd)$ \\
CoeF-D & 7.23& $2(\mathrm{f}^{s}{+}\mathrm{b}^{s})$ & 144.29 & $O(Nd)$ \\
\bottomrule
\end{tabular}
\end{table}

\subsection{Communication and Computation Costs}
\label{app:cost}

Table~\ref{tab:comm} summarizes the communication and computation costs of all methods. Compared with AccSFL, CoeF-SFL requires more server computation and uses downlink communication, which AccSFL does not need at all. We argue that neither cost is a major limitation. The server computation of CoeF-D is 7.23 GFLOPs per sample, which is twice that of vanilla SFL and AccSFL because the server additionally computes the gradients of the round-initial smashed data. This overhead is a constant factor that does not grow with the number of local iterations compared to vanilla SFL, and it remains below the server computation of GAS (15.22 GFLOPs) and MU-SplitFed (10.06 GFLOPs). Since the server in SFL typically has far more computing power than the clients, such a constant increase is much easier to absorb than an increase in client computation. The downlink of CoeF-D carries the same volume as that of vanilla SFL (144.29\,MB), so it adds no traffic beyond what gradient-based SFL already requires. It is, however, delivered in a single message per round instead of one message per mini-batch, which reduces the number of transfers from 64 to 2 and removes the synchronization between the client and the server during local training.

Tables~\ref{tab:auxsweep} and~\ref{tab:auxsweep_lang} indicate that the auxiliary network can be a more restrictive constraint than these costs. AccSFL assumes a small auxiliary network \citep{han2022localloss}, but even an FC layer without hidden layers amounts to 34.9\% of the trainable client-side parameters on ViT-Tiny and 13.9\% on DistilRoBERTa in our setting, which is a considerable overhead for the resource-limited clients that SFL targets. We vary the auxiliary network of AccSFL from an FC layer with LoRA adapters, which uses only 1.2\% of the trainable client-side parameters, to MLPs with up to 12 hidden layers, which exceed the trainable client-side parameters by a factor of 2.4 on ViT-Tiny and 17 on DistilRoBERTa. The accuracy of AccSFL increases with the size of the auxiliary network and then saturates. On ViT-Tiny/CIFAR-100, it rises from 63.6\% with LoRA to 81.9\% with eight hidden layers, which is still below the 84.5\% of CoeF-D. Even with the largest auxiliary network, AccSFL remains below CoeF-J in every setting and below CoeF-D in all settings except non-IID Food-101. On DistilRoBERTa, the accuracy saturates after a single hidden layer at about 68\% under IID data and 64\% under non-IID data, whereas CoeF-J reaches 69.8\% and 67.1\%. FSL-SAGE, whose auxiliary network is 25 times larger than the trainable client-side model on ViT-Tiny, also falls far behind CoeF-SFL, which shows that enlarging the auxiliary network alone does not close the gap.

Furthermore, these results show that the accuracy of auxiliary-network-based SFL is directly tied to the size of an auxiliary network that has to be stored and trained on the client, whose memory and computation are the scarcest resources in SFL. CoeF-SFL instead places its additional cost on the server computation and the downlink, which are less constrained, and achieves higher accuracy with less client computation and without any auxiliary network.

\begin{table}[ht]
\centering
\scriptsize
\setlength{\tabcolsep}{4pt}
\renewcommand{\arraystretch}{1.2}
\caption{Communication per client and round (top) and computation per sample (bottom) on ViT-Tiny/CIFAR-100 at cut 4. Transfers count uplink and downlink messages separately. For FSL-SAGE, entries with a slash give a round without alignment (left) and an alignment round (right, every $l=10$ rounds), in which the server recomputes the gradients of $N_{\mathrm{buf}}=64$ stored samples and trains the auxiliary network on them for $E=40$ epochs, with $\mathrm a$ counting the double backward of the alignment loss. The values are measured in our setting from Section \ref{sec:exp}.}
\label{tab:comm}
\begin{tabular}{l c c c c c}
\toprule
\textbf{Method} & \textbf{Exch.} & \textbf{Transfers} & \textbf{Volume} & \textbf{Up (MB)} & \textbf{Down (MB)} \\
\midrule
Vanilla SFL        & $O(B)$ & 64 & $O(Nd)$ & 144.29 & 144.29 \\
Gradient reuse (w/ perfect comp.) & $O(B)$ & 65 & $O(Nd)$ & 288.58 & 144.29 \\
Gradient reuse (w/o comp.)          & $O(1)$ & 2  & $O(Nd)$ & 144.29 & 144.29 \\
AccSFL             & $O(1)$ & 1  & $O(Nd)$ & 144.29 & 0.00 \\
MU-SplitFed        & $O(1)$ & 2  & $O(Nd)$ & 432.87 & 0.000122 \\
FSL-SAGE           & $O(1)$ & 1\,/\,2 & $O(Nd)$ & 144.29 & 0.00\,/\,5.32 \\
GAS                & $O(1)$ & 2  & $O(Nd)$ & 144.29 & 144.29 \\
\cmidrule(l){1-6}
CoeF-D (ours)      & $O(1)$ & 2  & $O(Nd)$ & 144.29 & 144.29 \\
CoeF-J (ours)      & $O(1)$ & 2  & $O((1+R)Nd)$ & 144.29 & 288.58 \\
\bottomrule
\end{tabular}

\vspace{6pt}
\begin{tabular}{l c c c c}
\toprule
 & \multicolumn{2}{c}{\textbf{Client}} & \multicolumn{2}{c}{\textbf{Server}} \\
\cmidrule(lr){2-3}\cmidrule(lr){4-5}
\textbf{Method} & \textbf{Complexity} & \textbf{GFLOPs} & \textbf{Complexity} & \textbf{GFLOPs} \\
\midrule
Vanilla SFL        & $\mathrm{f}^{c}{+}\mathrm{b}^{c}$ & 1.88 & $\mathrm{f}^{s}{+}\mathrm{b}^{s}$ & 3.63 \\
Gradient reuse (w/ perfect comp.) & $2\mathrm{f}^{c}{+}\mathrm{b}^{c}$ & 2.77 & $2(\mathrm{f}^{s}{+}\mathrm{b}^{s})$ & 7.23 \\
Gradient reuse (w/o comp.)          & $2\mathrm{f}^{c}{+}\mathrm{b}^{c}$ & 2.77 & $2(\mathrm{f}^{s}{+}\mathrm{b}^{s})$ & 7.23 \\
AccSFL             & $2\mathrm{f}^{c}{+}\mathrm{b}^{c}{+}\mathrm{a}$ & 2.80 & $\mathrm{f}^{s}{+}\mathrm{b}^{s}$ & 3.63 \\
MU-SplitFed        & $(1{+}2P)\mathrm{f}^{c}$ & 2.69 & $2P(\tau{+}1)\mathrm{f}^{s}$ & 10.06 \\
FSL-SAGE           & $2\mathrm{f}^{c}{+}\mathrm{b}^{c}$ & 2.77 & $\mathrm{f}^{s}{+}\mathrm{b}^{s}\,/\,(1{+}\tfrac{N_{\mathrm{buf}}}{N})(\mathrm{f}^{s}{+}\mathrm{b}^{s}){+}\tfrac{EN_{\mathrm{buf}}}{N}\mathrm{a}$ & 3.63\,/\,14.21 \\
GAS                & $2\mathrm{f}^{c}{+}\mathrm{b}^{c}$ & 2.77 & $(2{+}\tfrac{N_g}{N})(\mathrm{f}^{s}{+}\mathrm{b}^{s})$ & 15.22 \\
\cmidrule(l){1-5}
CoeF-D (ours)      & $2\mathrm{f}^{c}{+}\mathrm{b}^{c}{+}O(d)$ & 2.77 & $2(\mathrm{f}^{s}{+}\mathrm{b}^{s})$ & 7.23 \\
CoeF-J (ours)      & $2\mathrm{f}^{c}{+}\mathrm{b}^{c}{+}O(Rd)$ & 2.77 & $2\mathrm{f}^{s}{+}(2{+}R)\mathrm{b}^{s}$ & 9.14 \\
\bottomrule
\end{tabular}
\vspace{-8pt}
\end{table}

\begin{table}[t]
\centering
\scriptsize
\setlength{\tabcolsep}{1.6pt}
\renewcommand{\arraystretch}{1.05}
\caption{Auxiliary-network size sweep for AccSFL on ViT-Tiny at cut 4: server-side test accuracy (\%). Aux.\ params gives the trainable parameters of the auxiliary network and their ratio to the trainable client-side parameters (55.3K under LoRA). FC is the auxiliary network of AccSFL in the other tables, and FC + LoRA freezes the FC weight and trains rank-$r$ adapters and the bias. MLP stacks the given number of hidden layers of width $d/2=96$ with GELU. FSL-SAGE uses its own auxiliary network, which consists of three server-side blocks followed by a layer norm and an FC head.}
\label{tab:auxsweep}
\begin{tabular}{l rr cccccc}
\toprule
 & \multicolumn{2}{c}{\textbf{Aux.\ params}} & \multicolumn{3}{c}{\textbf{IID}} & \multicolumn{3}{c}{\textbf{Non-IID}}\\
\cmidrule(lr){2-3}\cmidrule(lr){4-6}\cmidrule(lr){7-9}
\textbf{Method} & \# & \% client & CIFAR-100 & Food-101 & Tiny-ImageNet & CIFAR-100 & Food-101 & Tiny-ImageNet\\
\midrule
FC + LoRA ($r{=}2$) & 0.7K & 1.2 & 63.6\std{1.2} & 64.9\std{0.4} & 48.8\std{0.7} & 36.6\std{0.6} & 37.0\std{4.5} & 24.1\std{2.7}\\
FC + LoRA ($r{=}4$) & 1.3K & 2.3 & 63.8\std{1.0} & 65.2\std{0.5} & 49.4\std{1.1} & 37.1\std{0.8} & 37.5\std{4.6} & 24.5\std{2.2}\\
FC (AccSFL) & 19.3K & 34.9 & 65.6\std{1.2} & 67.5\std{0.1} & 47.4\std{0.2} & 28.2\std{1.3} & 32.5\std{4.3} & 18.8\std{2.1}\\
MLP, 1 hidden & 28.2K & 51.0 & 78.8\std{0.2} & 77.5\std{0.2} & 70.7\std{0.2} & 54.7\std{1.9} & 51.5\std{3.8} & 42.0\std{2.8}\\
MLP, 4 hidden & 56.2K & 101.6 & 81.5\std{0.4} & 79.9\std{0.2} & 74.4\std{0.1} & 62.7\std{1.9} & 56.9\std{3.2} & 51.4\std{3.2}\\
MLP, 8 hidden & 93.4K & 168.9 & 81.9\std{0.1} & 80.6\std{0.1} & 75.4\std{0.2} & 64.6\std{2.2} & 58.7\std{3.7} & 53.9\std{3.6}\\
MLP, 12 hidden & 130.7K & 236.3 & 81.8\std{0.2} & 80.7\std{0.1} & 75.5\std{0.2} & 64.8\std{1.8} & 59.2\std{3.4} & 54.2\std{3.3}\\
\cmidrule(l){1-9}
FSL-SAGE & 1.40M & 2524.1 & 66.9\std{7.2} & 60.5\std{0.9} & 50.9\std{9.0} & 33.6\std{2.0} & 30.3\std{5.5} & 20.0\std{2.3}\\
\cmidrule(l){1-9}
CoeF-D (ours) & 0 & 0.0 & 84.5\std{0.2} & 83.4\std{0.2} & 78.3\std{0.1} & 67.3\std{1.5} & 53.5\std{2.4} & 54.9\std{2.7}\\
CoeF-J (ours) & 0 & 0.0 & 83.9\std{0.3} & 83.6\std{0.0} & 77.8\std{0.1} & 69.1\std{1.2} & 62.9\std{3.2} & 57.7\std{3.4}\\
\bottomrule
\end{tabular}
\end{table}

\begin{table}[t]
\centering
\scriptsize
\setlength{\tabcolsep}{3pt}
\renewcommand{\arraystretch}{1.05}
\caption{Auxiliary-network size sweep on DistilRoBERTa/20 Newsgroups at cut 2: server-side test accuracy (\%). Columns as in Table~\ref{tab:auxsweep}, where the client-side model has 110.6K trainable parameters (LoRA) and MLP uses hidden width $d/2=384$. FSL-SAGE uses its own auxiliary network, which consists of two of the four server-side blocks followed by a layer norm and an FC head, with 14.3M parameters.}
\label{tab:auxsweep_lang}
\begin{tabular}{l rr cc}
\toprule
 & \multicolumn{2}{c}{\textbf{Aux.\ params}} & & \\
\cmidrule(lr){2-3}
\textbf{Method} & \# & \% client & \textbf{IID} & \textbf{Non-IID}\\
\midrule
FC + LoRA ($r{=}2$) & 1.6K & 1.4 & 65.3\std{0.6} & 59.7\std{1.9}\\
FC + LoRA ($r{=}4$) & 3.2K & 2.9 & 65.4\std{0.5} & 59.8\std{1.7}\\
FC (AccSFL) & 15.4K & 13.9 & 66.6\std{0.3} & 61.1\std{1.5}\\
MLP, 1 hidden & 303.0K & 274.0 & 68.1\std{0.1} & 64.0\std{1.1}\\
MLP, 4 hidden & 746.5K & 675.0 & 67.9\std{0.0} & 64.0\std{0.6}\\
MLP, 8 hidden & 1.34M & 1209.7 & 67.8\std{0.8} & 64.4\std{1.1}\\
MLP, 12 hidden & 1.93M & 1744.5 & 67.8\std{0.7} & 64.1\std{1.5}\\
\cmidrule(l){1-5}
FSL-SAGE & 14.3M & 115.3 & 62.0\std{0.8} & 53.2\std{1.6}\\
\cmidrule(l){1-5}
CoeF-D (ours) & 0 & 0.0 & 69.7\std{0.1} & 65.7\std{1.0}\\
CoeF-J (ours) & 0 & 0.0 & 69.8\std{0.3} & 67.1\std{0.6}\\
\bottomrule
\end{tabular}
\end{table}

\end{document}